\documentclass[letterpaper, 10 pt, conference]{ieeeconf}  
\usepackage{hyperref}
\usepackage{amsmath,amsfonts}
\usepackage{algorithmic}
\usepackage{algorithm}
\usepackage{array}
\usepackage[caption=false,font=normalsize,labelfont=sf,textfont=sf]{subfig}
\usepackage{textcomp}
\usepackage{stfloats}
\usepackage{url}
\usepackage{verbatim}
\usepackage{graphicx}
\usepackage{cite}
\usepackage{amssymb}
\usepackage{makecell}

\newtheorem{theorem}{\bf Theorem}
\newtheorem{lemma}[theorem]{\bf Lemma}
\newtheorem{definition}[theorem]{\bf Definition}

\newtheorem{corollary}[theorem]{\bf Corollary}
\IEEEoverridecommandlockouts 

\title{\LARGE \bf
CDT-Dijkstra: Fast Planning of Globally Optimal Paths for All Points in 2D Continuous Space
}

\author{Jinyuan Liu, Minglei Fu, Wenan Zhang,~\IEEEmembership{Member,~IEEE,} Bo Chen,~\IEEEmembership{Member,~IEEE,} \\Ryhor Prakapovich, and Uladzislau Sychou
\thanks{This work was supported by the National Key Research and Development Program of China under Grant No. 2022YFE0121700.}
\thanks{Jinyuan Liu, Minglei Fu, Wenan Zhang, and Bo Chen are with the College of Information Engineering, Zhejiang University of Technology, Hangzhou, 310023, China.}%
\thanks{Ryhor Prakapovich and Uladzislau Sychou are with the the United Institute of Informatics Problems of the National Academy of Sciences of Belarus, Minsk, 220012, Belarus.}%
\thanks{Minglei Fu is the corresponding author, and the phone: +86-571-85292552; fax: 86-571-85292552; e-mail: fuml@zjut.edu.cn.}%
}

\begin{document}
\maketitle
\thispagestyle{empty}
\pagestyle{empty}


\begin{abstract}

The Dijkstra algorithm is a classic path planning method, which in a discrete graph space, can start from a specified source node and find the shortest path between the source node and all other nodes in the graph. However, to the best of our knowledge, there is no effective method that achieves a function similar to that of  the Dijkstra's algorithm in a continuous space. In this study, an optimal path planning algorithm called convex dissection topology (CDT)-Dijkstra is developed, which can quickly compute the global optimal path from one point to all other points in a 2D continuous space. CDT-Dijkstra is mainly divided into two stages: SetInit and GetGoal. In SetInit, the algorithm can quickly obtain the optimal CDT encoding set of all the cut lines based on the initial point $x_{init}$. In GetGoal, the algorithm can return the global optimal path of any goal point at an extremely high speed. In this study, we propose and prove the planning principle of considering only the points on the cutlines, thus reducing the state space of the distance optimal path planning task from 2D to 1D. In addition, we propose a fast method to find the optimal path in a homogeneous class and theoretically prove the correctness of the method. Finally, by testing in a series of environments, the experimental results demonstrate that CDT-Dijkstra not only plans the optimal path from all points at once, but also has a significant advantage over advanced algorithms considering certain complex tasks.

\end{abstract}


\section{INTRODUCTION} 
Considering the continuous progress in science and technology in recent years, an unprecedented rapid development of technologies has been achieved in the field of robotics \cite{wang2020overview}. As an important research direction in artificial intelligence, robotics, and automatic control, the path-planning method aims to find a continuous trajectory in the state space for robots that connect the initial and target states \cite{zhang2018multilevel,bopardikar2015multiobjective}. Path-planning algorithms are now widely used in various fields, such as autonomous driving \cite{duhautbout2022efficient}, mobile robotics \cite{tzafestas2018mobile}, game development \cite{naderi2015rt}, smart homes \cite{do2022heat}, and biomedical applications \cite{zhao2022surgical}.

Graph search-based algorithms are widely used in path-planning. These algorithms share the common characteristic of using graph models to represent the environment or search space and to find the shortest or optimal path in a graph. Dijkstra's algorithm \cite{dijkstra1959note} is a breadth-first search algorithm that can find the shortest path from each node in the graph to the source node. The A* algorithm \cite{hart1968formal} is an improvement of the Dijkstra algorithm, which incorporates a heuristic function to improve the search efficiency while ensuring the shortest path. Graph search-based path-planning algorithms are often stable and perform well; however, their search process is conducted on a discrete graph, making it difficult to solve optimal path-planning problems in continuous spaces. While \cite{lozano1979algorithm} achieves the utilization of graph search in continuous space by visibility graph, such methods are not efficient in complex environments, cause the search for visible points becomes hard, and the visibility graph can become large.
 

Randomized sampling-based algorithms have become a popular research topic in the field of motion planning owing to their ability to avoid discretization problems of graph search techniques. The rapidly-exploring random Tree (RRT) algorithm \cite{lavalle2001randomized} is the most commonly used sampling-based algorithm. Subsequently, researchers have proposed various improvements to the RRT algorithm. The RRT* algorithm incorporates a rewiring process into the RRT sampling process, ensuring both probabilistic completeness and asymptotic optimality \cite{karaman2011sampling}. The informed-RRT* algorithm was proposed for path planning in complex environments, and it uses a non-uniform sampling strategy to converge to the optimal path more quickly \cite{gammell2018informed}. The batched informed trees (BIT*) \cite{gammell2020batch} and adaptively informed trees (AIT*) \cite{strub2022adaptively} algorithms combine the advantages of graph search algorithms and sampling-based algorithms, use heuristics for all aspects of path cost to prioritize the search for highquality paths and focus on improvements. 

Based on cell decomposition, the path planning method is highly effective in addressing road connectivity and rapid planning problems in static environments. These methods mainly involve dividing the map into a series of connected regions or polygons, connecting these regions or polygons to form a connected graph, and usinge the connection relationship of cells to determine a collision-free path that the robot can run \cite{kloetzer2015optimizing,li2020new}. However, path-planning algorithms based on cell decomposition only consider connectivity and ignore other properties, such as optimality (e.g., path length) and computational complexity (e.g., time for partitioning the environment into a finite set of regions) \cite{gonzalez2017comparative}. This often requires the assistance of other algorithms to obtain a globally optimal path. For example, methods referenced in previous studies created heuristic topological maps using generalised Voronoi diagrams (GVD), which then guide the sampling process of RRT to ensure that the algorithm can find the globally optimal path \cite{chi2021generalized,huang2021path}; these heuristic paths are not inevitably optimal and the GVD initialization time for complex maps is extensive.

In recent years, the concept of path homotopy has received significant attention in the field of path planning. Homotopy is a concept in topology \cite{munkres2018elements} that refers to the ability of two paths in a continuous space to transform into one another without colliding with obstacles. In path planning, homotopy can be used to verify whether two paths are topologically equivalent, and obstacles in the environment divide the possible path space of the robot into a countable number of homotopy classes. These homotopy classes can serve as a high-level abstraction of robot paths and are a powerful tool for solving robot planning tasks \cite{bhattacharya2012search,mccammon2021topological,wakulicz2023topological}. Homotopy invariants are labels used to identify the homotopy class to which a path belongs. In a 2D space, the h-signature is commonly used as a homotopy invariant. However, the h-signature does not directly record the connectivity of the space; therefore, it cannot guide the path-planning process. In addition, it is relatively difficult to deduce the reference path in a homotopy class based on the h-signature. In our previous study \cite{arxivLiu2023homotopy}, we developed a more effective homotopy invariant, this is, a homotopy path class encoder based on convex dissection topology (CDT). The CDT Encoder can easily encode paths in a homotopy class to the same CDT encoding, and the CDT encoding can also be quickly decoded into a reference trajectory in the homotopy class. The CDT Encoder constructs a topological graph that contains the connectivity information of the environmental space, which allows it to effectively guide the path planning process.

In the previous afore \cite{arxivLiu2023homotopy} mentioned studies, graph-search-based planning methods have difficulty in solving the problem of optimal path planning in a continuous space. The cell decomposition-based method is commonly used to inspire sampling-based planning methods; however, it is difficult to ensure that these heuristic paths are optimal. In response to these issues, this study presents the highly challenging path-planning task of efficiently planning a globally optimal path from one point to all other points in a 2D continuous space.\footnote{The description here is not rigorous because there are infinitely many points in the continuous space. Thus, if an algorithm can very rapidly return the optimal path from any goal point to the initial point after processing the initial point, it is considered compliant with this description.}

In this study, we mainly investigated the aforementioned task by combining the relevant theories of cell decomposition, graphs, and homotopy invariants. The contributions of our study can be summarized as follows:
\begin{itemize}
    \item{Based on the theoretical foundation of map convex division \cite{arxivLiu2023homotopy}, we further proposed a planning principle that only considers the points on the cutting lines, which reduces the state space of the optimal path planning task from 2D to 1D.}
    \item{A method for efficiently finding the optimal path in the homotopy class in a map partitioned by convex division is proposed.}
    \item{an algorithm called CDT-Dijkstra  was proposed, which can quickly plan globally-optimal paths from one point to all other points in a 2D continuous space.}
\end{itemize}

The problem formulation, relevant mathematical definitions, and proofs are presented in Section II. Details of our algorithm are discussed in Section III. The experimental results are presented and analyzed in Section IV, which are followed by the conclusions in Section V.

\section{PRELIMINARIES}
This study was mainly based on the theoretical foundation of our previous work \cite{arxivLiu2023homotopy}. In this study, we directly adopt the relevant mathematical expressions and theoretical foundations presented in our previous study \cite{arxivLiu2023homotopy} without redundancy. Further details regarding the referenced study can be found in Sections II and III of \cite{arxivLiu2023homotopy}.\footnote{Since \cite{arxivLiu2023homotopy} is still under review, we temporarily publicise its implementation to verify the correctness of the method presented in this paper, at \url{https://arxiv.org/abs/2302.13026}. In fact, the mathematical representation of paths in this paper is consistent with the standard representation in topology.}

\subsection{Path Planning Problem Formulation}
This study investigated the problem of distance-optimized path planning for mobile robots in a 2D bounded space. The problem is defined as follows:

Let $X\subset\mathbb{R} ^2$ be the state space; the obstacle space and free space are denoted as $X_{obs}$ and $X_{free}=X/X_{obs}$, respectively. The path planning algorithm aims to determine a feasible path $f(t)$ such that
\begin{equation}
    \label{eq_III1}
    f\in P(X_{free};x_{init},x_{goal}),
\end{equation}
where $x_{init}$ denotes the initial state and $x_{goal}$ denotes the goal state. Let $P_f=P(X_{free};x_{init},x_{goal})$. The optimal path\footnote{In this study, $f^\circledast$ represents the global optimal path with the same start and end points as $f$, and $f^*$ represents the local optimal path in $[f]$.} planning problem can be defined as follows:
\begin{equation}
    \label{eq_III2}
    f^\circledast = \arg \min _{f \in P_{f}} S(f),
\end{equation}
where $S:P(X)\to \mathbb{R}_{\geqslant 0}$ is a function of the path length, which is defined as follows:
\begin{equation}
    \label{eq_III3}
    S(f)=\lim _{n\to N^+}\sum_{t=0}^{n}\|f(t+n^{-1})-f(t)\|,
\end{equation}
where $\|\cdot\|$ is the 2-Norm in Euclidean space and $N^+$ is a sufficiently large positive integer.

Compared to typical path planning tasks that only search for the optimal path between two given points, this study focuses on how to quickly find the optimal paths from $x_{init}$ to all the points in a space, given the initial point $x_{init}$.

\subsection{Dimensionality Reduction of State Space}
\begin{lemma}
    \label{NSCLOP} 
    A necessary and sufficient condition for $f^* \in P(X_{free})$ to be the shortest path in $[f^*]$ is that, for any $t_1,t_2 \in [0,1]$, there exists the shortest path $g(t)=f^*(t_1+(t_2-t_1)t)$ in $[g]$.
\end{lemma}
\begin{proof}

    Necessary: The law of proof by contradiction is used herein. Suppose $f^*$ is the shortest path in $[f^*]$, and $\exists t_1, t_2 \in [0,1]$ such that $g(t)=f^*(t_1+(t_2-t_1)t)$ is not the shortest path in $[g]$; that is, there exists $g' \simeq_p g$ such that $S(g')<S(g)$. The path $f^*$ can be expressed as $f^* \cong_p g_s*g*g_e$, where
    \begin{align}
        \label{eq_NSCLOP1}
         & g_s(t) = f^*(a \cdot t),   & a = \min (t_1,t_2), \\
         & g_e(t) = f^*(b + (1-b) t), & b = \max (t_1,t_2).
    \end{align}
    Hence,
    \begin{equation}
        \label{eq_NSCLOP2}
        \begin{split}
            S(f^*) &= S(g_s) + S(g) + S(g_e)\\
            &< S(g_s) + S(g') + S(g_e)\\
            &< S(g_s * g' * g_e).
        \end{split}
    \end{equation}
    Therefore, there exists a path $g_s * g' * g_e$ in $[f^*]$ that is shorter than $f^*$. This contradicts our original hypothesis. Therefore, the sufficiency condition is true.

    Sufficient: When the latter of the proposition holds, let $t_1=0$ and $t_2=1$. At this moment $g(t)=f^*(t)$. Therefore, $f^*$ is the shortest path in $[f^*]$. Therefore, the necessary condition is true.
\end{proof}
\begin{theorem}
    \label{SSDR} 
    For any optimal path $f^* \in P(X_{free})$ where $x_s=f^*(0)$, $x_e=f^*(1)$. The $f^*$ must be expressed in the form of multiple line splices as follows.
    When $x_s$ and $x_e$ are within the same convex polygon,
    \begin{equation}
        \label{eq_SSDR1}
        f^* \cong_p l^{x_e}_{x_s}.
    \end{equation}
    When $x_s$ and $x_e$ are not within the same convex polygon,
    \begin{equation}
        \label{eq_SSDR2}
        f^* \cong_p l_{x_0}^{x_1} * l_{x_1}^{x_2} * \cdots * l_{x_{n-1}}^{x_n},
    \end{equation}
    where $x_0=x_s$, $x_n=x_e$, $x_1,\dotsc,x_{n-1}$ is the intersection of $f^*$ in turn with the cutlines.
\end{theorem}
\begin{proof}
    According to the properties of convex polygons, (\ref{eq_SSDR1}) is established, and according to the characteristics of the convex dissection method used in \cite{arxivLiu2023homotopy} (refer to Subsection III.A of \cite{arxivLiu2023homotopy}), the points $x_0,x_1,\cdots ,x_n$ in (\ref{eq_SSDR2}) exist. Therefore, $f^*$ can be expressed as follows:
    \begin{equation}
        \label{eq_SSDR3}
        f^* \cong_p f_1 * f_2 * \cdots * f_n ,
    \end{equation}
    where $f_k \in P(X_{free},x_{k-1},x_k)$. Moreover, based on \textbf{Lemma~\ref{NSCLOP}}, the optimal path $f^*$ must guarantee that each of its segments is optimal; hence, for any $f_k$, $f_k = l^{x_k}_{x_{k-1}}$.
\end{proof}

According to \textbf{Theorem~\ref{SSDR}}, we only need to consider the points on the cutlines. The state space of the optimal path planning task is reduced from the 2D space of $X_{free}$ to a 1D space on the cutlines.

\subsection{Problem Simplification}
In this subsection, we simplify the problem addressed in this study. Under the guidance of \textbf{Lemma~\ref{NSCLOP}} and \textbf{Theorem~\ref{SSDR}}, we can easily obtain the following corollary:

\begin{corollary}
    \label{CROP} 
    When $x_s$ and $x_e$ are not within the same convex polygon, there must exist a point $x_c$ on its adjacent\footnote{We define the cutlines adjacent to a point as the cutlines of the convex polygon it belongs to, and consider the points on the cutline as belonging to two convex polygons simultaneously.} cutlines such that the optimal path $f^\circledast_e$ from $x_s$ to $x_e$ can be expressed as follows:
    \begin{equation}
        \label{eq_CROP1}
        f^\circledast_e \cong_p f^\circledast_c * l^{x_e}_{x_c},
    \end{equation}
    where $f^\circledast_c$ is the optimal path from $x_s$ to $x_c$.
\end{corollary}

Based on the afore mentioned, the optimal path $f^\circledast_{e}$ and $f^\circledast_c$ share the same CDT encoding in \textbf{Corollary~\ref{CROP}}; that is, $\Gamma \circ f^\circledast_e \cong_p \Gamma \circ f^\circledast_c$. Thus we can derive a new constraint for the global optimal path, that is, when $x_{goal}$ and $x_{init}$ are not within the same convex polygon, the following is obtained:
\begin{equation}
    \label{eq_III4}
    \Gamma \circ f^\circledast \in \Gamma \circ \bigcup_{k} C_k^\circledast = \bigcup_{k} (\Gamma \circ C_k^\circledast).
\end{equation}
Here, $C_k^\circledast$ represents the set of optimal paths for the all points on each adjacent cutline of $x_{goal}$. Note, the set $C_k^\circledast$ has an infinite number of elements, but its CDT encoding set $\Gamma \circ C_k^\circledast$ has a finite number of elements. Therefore, for a given $x_{init}$, if we can obtain the optimal CDT encoding set of all the cutlines, it is easy to obtain the global optimal path from any $x_{goal}$ to $x_{init}$. These optimal CDT encoding sets of cutlines can be obtained by iterating outward from the adjacent cutlines of $x_{init}$ using (\ref{eq_CROP1}) and (\ref{eq_III4}). For convenience, we use $\xi_{c_k}$ to represent the optimal CDT encoding set $\Gamma \circ C_k^\circledast$ corresponding to cutline $c_k$ in the following.

\section{APPROACH}
\begin{figure*}[!h]
    \centering
    \subfloat[]{\includegraphics[width=1.2in]{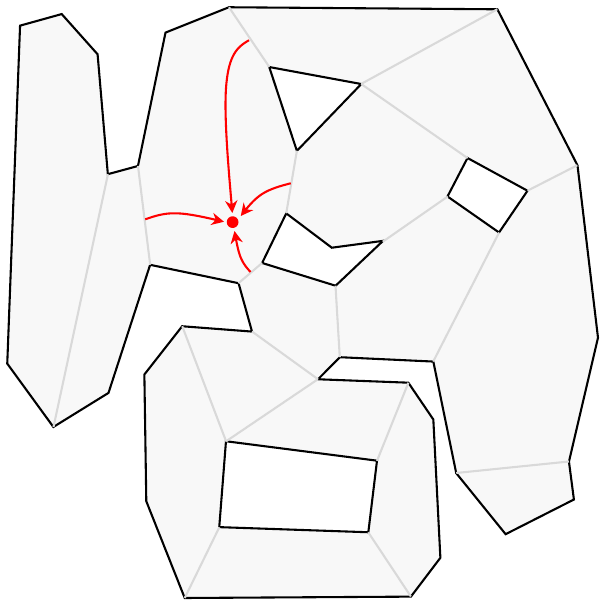}
        \label{fig_process:map1}}
    \hfil
    \subfloat[]{\includegraphics[width=1.2in]{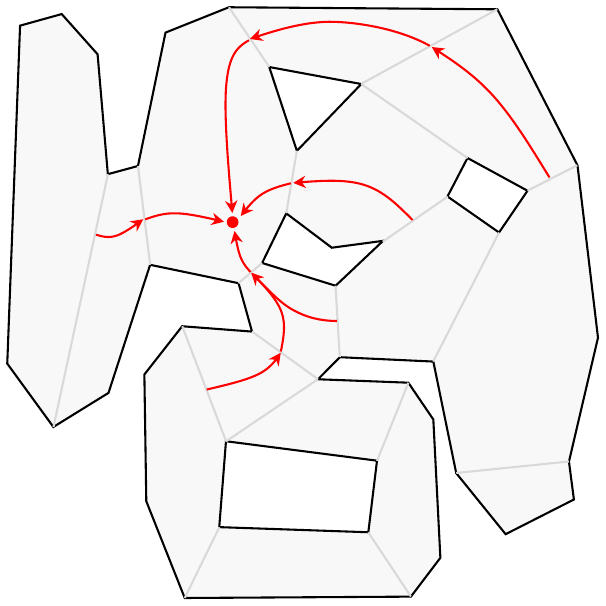}
        \label{fig_process:map2}}
    \hfil
    \subfloat[]{\includegraphics[width=1.2in]{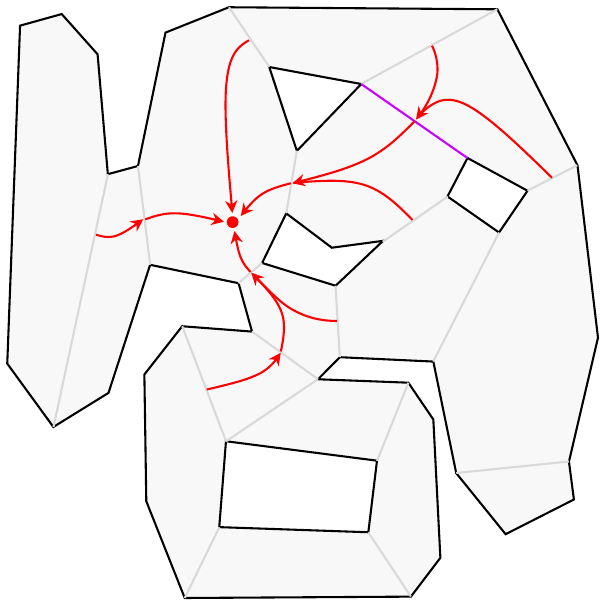}
        \label{fig_process:map3}}
    \hfil
    \subfloat[]{\includegraphics[width=1.2in]{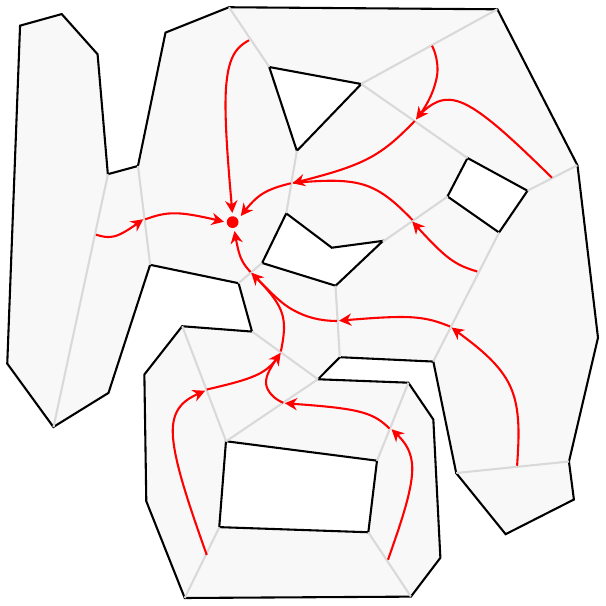}
        \label{fig_process:test1}}
    \hfil
    \subfloat[]{\includegraphics[width=1.2in]{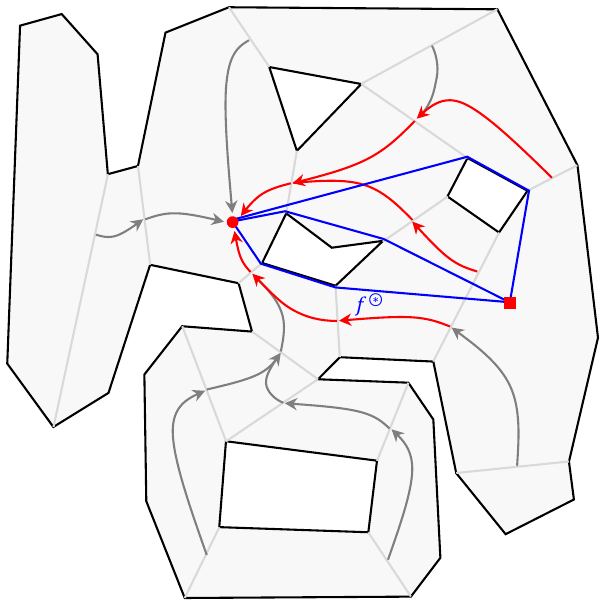}
        \label{fig_process:test2}}
    \hfil
    \caption{Visual illustration of CDT-Dijkstra at work. Firstly, (a) create the CDT encoding sets of the cutlines adjacent to the initial point, followed by (b) gradually updating the CDT encoding sets for other cutlines outward. (c) The diagram of Algorithm~\ref{alg:alg3}, when the red line is added, it affects the CDT encoding set of its adjacent cutline. (d) The computation of CDT encoding set for all cutlines is completed. (e) All valid CDT encodings adjacent the goal point are found and the globally optimal path is obtained.}
    \label{fig_process}
\end{figure*}

In this subsection, we propose and explain the details of a unique optimal path planning algorithm, CDT-Dijkstra, based on a convex division topological graph. The characteristic of this algorithm is as follows: (i) given an initial point $x_{init}$, it can quickly obtain the optimal CDT encoding set for all the cutlines, and (ii) for any goal point $x_{goal}$ in the free space $X_{free}$, it can return the optimal path from $x_{init}$ to $x_{goal}$ at a very high speed. The CDT corresponding to its characteristics can be divided into two independent stages (threads), which are namely SetInit and GetGoal. A visualization of CDT-Dijkstra at work is shown in Fig.~\ref{fig_map}.

\subsection{Main Algorithm}
The main part of the CDT-Dijkstra is shown in Algorithm~\ref{alg:alg1}. For lines 1-2 of the algorithm we use the method in \cite{arxivLiu2023homotopy} to construct the topological graph $X_{tp}$ corresponding to the free space $X_{free}$ based on the input $Map$. In line 3, the optimal CDT encoding set for each cutline will be created. Thereafter, the algorithm enters the SetInit stage (lines 4-19).

In the SetInit stage, the algorithm obtains the initial point of the planning task (line 5) and updates the sets $\{\xi_{c_1},\xi_{c_2},\cdots ,\xi_{c_n} \}$ when $x_{init}$ changes (lines 6-18). To introduce the pseudocode of the update process, we define a special class of symbol $\hat{\mathrm{f}}_{c}$ used in the pseudocode.
\begin{definition}[Symbol $\hat{\mathrm{f}}_{c}$]
    \label{Symbol}
    Let $c$ be the cutline connecting two convex polygons, denoted by $X_{\mathrm{x}_1}$ and $X_{\mathrm{x}_2}$, respectively. The corresponding nodes of $X_{\mathrm{x}_1}$ and $X_{\mathrm{x}_2}$ in $X_{tp}$ are denoted by $\mathrm{x}_1$ and $\mathrm{x}_2$, respectively. The symbol $\hat{\mathrm{f}}_{c}$ represents both $(\mathrm{x}_1,\mathrm{x}_2)$ and $(\mathrm{x}_2,\mathrm{x}_1)$ in $P(X_{tp})$. When $\hat{\mathrm{f}}_{c}$ is involved in the $*$ operation, it uses the form that makes the operation meaningful.
\end{definition}

During the update process, all encoding sets are first emptied. Thereafter, the encoding sets for the cutlines adjacent to $x_{init}$ are set, and a set $\vartheta_{ex}$ for recording the existing cutlines is created (lines 8-10). In line 11 of the algorithm, a list $\vartheta_{add}$ is created to record the cutlines that can be added to $\vartheta_{ex}$.

The algorithm iterates in lines 12-18. During each iteration, the algorithm pops a cutline $c_k$ from $\vartheta_{add}$ and uses Algorithm~\ref{alg:alg2} to update the encoding sets $\xi_{c_k}$ corresponding to $c_k$ (lines 13,14). The algorithm then pushes $c_k$ into $\vartheta_{ex}$, and pushes the unused cutlines adjacent to $\xi_{c_k}$ into $\vartheta_{add}$. Similar to the rewiring of RRT*, the change in the optimal homotopy class code in $\xi_{x_k}$ has an impact on other encoding sets. Therefore, at the end of each iteration, we optimized the existing encoding sets using Algorithm~\ref{alg:alg3}. When $\vartheta_{add}$ is empty, the update to sets $\xi_{c_1},\xi_{c_2},\cdots ,\xi_{c_n} $ is complete.
\begin{algorithm}[H]
    \caption{CDT-Dijkstra.}\label{alg:alg1}
    \begin{algorithmic}[1]
        \STATE {\textbf{Input: }}$Map$
        \STATE $X_{tp} \gets $BuildTopologyGraph$(Map)$
        \STATE Initialise $\{\xi_{c_1},\xi_{c_2},\cdots ,\xi_{c_n} \}$ with the number of cutlines
        \STATE \textbf{loop}
        \STATE \hspace{0.5cm} $x_{init} \gets$ WaitingTaskInput()
        \STATE \hspace{0.5cm} \textbf{if} $x_{init}$ is changed \textbf{then}
        \STATE \hspace{1.0cm} Fill $\{\xi_{c_1},\xi_{c_2},\cdots \xi_{c_n}\}$ with $\varnothing$
        \STATE \hspace{1.0cm} $\vartheta _{ex} \gets$ AdjacentCutlines$(x_{init})$
        \STATE \hspace{1.0cm} \textbf{for $c_k \in \vartheta _{ex} $ do}
        \STATE \hspace{1.5cm} push $(\Gamma \circ e_{x_{init}}) * \hat{\mathrm{f}}_{c_k} $ to $\xi_{c_k}$
        \STATE \hspace{1.0cm} $\vartheta _{add} \gets$ AdjacentCutlines$(\vartheta _{ex})$
        \STATE \hspace{1.0cm} \textbf{repeat}
        \STATE \hspace{1.5cm} $c_k \gets$ PopFirst$(\vartheta _{add})$
        \STATE \hspace{1.5cm} $\xi_{c_k} \gets$ GetHomotopyClassesForCutline$(c_k)$
        \STATE \hspace{1.5cm} Push AdjacentCutlines$(c_k)/ \vartheta _{ex}$ to $\vartheta _{add}$
        \STATE \hspace{1.5cm} Push $c_k$ to $\vartheta _{ex}$
        \STATE \hspace{1.5cm} OptimizedExistingHomotopyClasses$(c_k)$
        \STATE \hspace{1.0cm} \textbf{until} $\vartheta _{add}$ is empty
        \STATE \textbf{end loop}
    \end{algorithmic}
    \label{alg1}
\end{algorithm}

After the SetInit stage, CDT-Dijkstra can enter the GetGoal stage, and the global optimal path from $x_{init}$ to $x_{goal}$ can be quickly obtained using Algorithm~\ref{alg:alg5}. In Algorithm~\ref{alg:alg5}, if $x_{init}$ and $x_{goal}$ are in the same convex polygon, the algorithm directly returns $l^{x_{goal}}_{x_{init}}$ (lines 2,3). Otherwise, the algorithm finds all the CDT encoding on cutlines adjacent to $x_{goal}$ (lines 4-7). Finally, these CDT encodings are mapped into the path space $P(X_{free};x_{init},x_{goal})$ using Algorithm~\ref{alg:alg4}, and the global optimal path $f^\circledast$ is determined (lines 8-14). In line 10 of the algorithm, we need to determine whether $x_{goal}$ is in the convex polygon corresponding to the last element of $\mathrm{f}$ because the cutline belongs to two different convex polygons at the same time, and its encoding set may record the CDT encoding to the other convex polygon.
\begin{algorithm}[H]
    \caption{GetGoal.}\label{alg:alg5}
    \begin{algorithmic}[1]
        \STATE {\textbf{Input: }}$x_{goal}$
        \STATE \textbf{if} $x_{init}, x_{goal}$ in the same convex polygon \textbf{then}
        \STATE \hspace{0.5cm} \textbf{return} $l_{x_{init}}^{x_{goal}}$
        \STATE $\xi_{goal} \gets \varnothing$
        \STATE $\vartheta_{goal} \gets$ AdjacentCutlines$(x_{goal})$
        \STATE \textbf{for $c_k \in \vartheta_{goal}$ do}
        \STATE \hspace{0.5cm} $\xi_{goal} \gets \xi_{goal} \cup \xi_{c_k}$
        \STATE ${cost} \gets \infty$
        \STATE \textbf{for $\mathrm{f} \in \xi_{goal}$ do}
        \STATE \hspace{0.5cm} \textbf{if} $x_{goal} \in \mathrm{f}(-1)$ \textbf{then}
        \STATE \hspace{1.0cm} $f^* \gets $ GetShortestPath$(\mathrm{f},x_{init},x_{goal})$
        \STATE \hspace{1.0cm} \textbf{if ${cost} > S(f^*)$ then}
        \STATE \hspace{1.5cm} $f^\circledast \gets f^*$
        \STATE \hspace{1.5cm} ${cost} \gets S(f^*)$
        \STATE \textbf{return} $f^\circledast$
    \end{algorithmic}
    \label{alg5}
\end{algorithm}

\subsection{Get Homotopy Classes For Cutline}
The function of Algorithm~\ref{alg:alg2} is to update the encoding set of the input cutline $c_k$ according to the current encoding sets of adjacent cutlines in $\vartheta_{ex}$. According to (\ref{eq_CROP1}) and (\ref{eq_III4}), the element in the encoding set of $c_k$ must be expressed as the product of the element in its adjacent encoding sets and $\hat{\mathrm{f}}_{c_k}$. In Algorithm~\ref{alg:alg2}, we first determine all the codes of the encoding sets adjacent to $c_k$ (lines 2,3). In lines 5-14 of Algorithm~\ref{alg:alg2}, $\xi_u$ is used to compute the optimal path for each point on $c_k$.\footnote{We will discretize $c_k$; when the interval is very small, the discretisation has no effect on the planning results.} In this process, if encoding $\mathrm{f}$ can make $\mathrm{f}*\hat{\mathrm{f}}_{c_k}$ meaningful and $\mathrm{f}*\hat{\mathrm{f}}_{c_k}$ can obtain the optimal path to a certain point (lines 8-14), then $\mathrm{f}*\hat{\mathrm{f}}_{c_k}$ is placed into an encoding set $\xi_{out}$.
\begin{algorithm}[H]
    \caption{GetHomotopyClassesForCutline.}\label{alg:alg2}
    \begin{algorithmic}[1]
        \STATE {\textbf{Input: }}$c_k$
        \STATE $\vartheta _{near} \gets$ AdjacentCutlines$(c_k)\cap \vartheta _{ex}$
        \STATE $\xi_{u} \gets \bigcup_{c_i \in \vartheta_{near}} \xi_{c_i}$
        \STATE $\xi_{out} \gets \varnothing$
        \STATE \textbf{for $t \in [0,1]$ do}
        \STATE \hspace{0.5cm} $x_t \gets c_k(t)$
        \STATE \hspace{0.5cm} ${cost}_{old} \gets \infty $
        \STATE \hspace{0.5cm} \textbf{for $\mathrm{f} \in \xi_{u}$ do}
        \STATE \hspace{1.0cm} \textbf{if} $\mathrm{f}*\hat{\mathrm{f}}_{c_k}$ is meaningful \textbf{then}
        \STATE \hspace{1.5cm} $f^* \gets$GetShortestPath$(\mathrm{f}*\hat{\mathrm{f}}_{c_k},x_{init},x_t)$
        \STATE \hspace{1.5cm} ${cost}_{new} \gets S(f^*)$
        \STATE \hspace{1.5cm} \textbf{if ${cost}_{old}>{cost}_{new}$ then}
        \STATE \hspace{2.0cm} $\mathrm{f}_{min} \gets \mathrm{f}*\hat{\mathrm{f}}_{c_k}$
        \STATE \hspace{2.0cm} ${cost}_{old} \gets {cost}_{new}$
        \STATE \hspace{0.5cm} $\xi_{out} \gets \xi_{out} \cup \{\mathrm{f}_{min}\}$
        \STATE \textbf{return} $\xi_{out}$
    \end{algorithmic}
    \label{alg2}
\end{algorithm}
\subsection{Optimized Existing Homotopy Classes}
The purpose and process of the optimization by Algorithm~\ref{alg:alg3} are similar to the rewiring operation of RRT*. First, Algorithm~\ref{alg:alg3} creates a list $\vartheta_{r}$ of cutlines to be optimized based on the input $c_k$, which subsequently starts the optimization of the encoding set of cutlines in $\vartheta_{r}$. The optimization of the encoding set of all the cutlines in $\vartheta_{ex}$ is complete when $\vartheta_{r}$ is empty (lines 3-9). In each optimization iteration, the algorithm pops a cutline $c_r$ to be optimized from $\vartheta_{r}$ and recomputes the encoding set of that cutline. If the encoding set of this cutline changes, the algorithm will push the cutlines adjacent to $c_r$ in $\vartheta_{ex}$ into $\vartheta_{r}$.
\begin{algorithm}[H]
    \caption{OptimizedExistingHomotopyClasses.}\label{alg:alg3}
    \begin{algorithmic}[1]
        \STATE {\textbf{Input: }}$c_k$
        \STATE $\vartheta_{r} \gets$ AdjacentCutlines$(c_k)\cap \vartheta _{ex}$
        \STATE \textbf{repeat}
        \STATE \hspace{0.5cm} $c_{r} \gets$ PopFirst$(\vartheta_{r})$
        \STATE \hspace{0.5cm} $\xi_{temp} \gets$ GetHomotopyClassesForCutline$(c_r)$
        \STATE \hspace{0.5cm} \textbf{if $\xi_{c_r} \neq \xi_{temp}$ then}
        \STATE \hspace{1.0cm} $\xi_{c_r} \gets \xi_{temp}$
        \STATE \hspace{1.0cm} Push AdjacentCutlines$(c_k)\cap \vartheta_{ex}$ to $\vartheta_{r}$
        \STATE \textbf{until} $\vartheta_{r}$ is empty
    \end{algorithmic}
    \label{alg3}
\end{algorithm}

\subsection{Get Shortest Path}
Algorithm~\ref{alg:alg4} is primarily used to fast obtain the shortest path in the homotopy path class. In line 2 of Algorithm~\ref{alg:alg4}, $\epsilon_1,\epsilon_2,\cdots ,\epsilon_m$ are the cutlines through which $\mathrm{f}$ passes subsequently. In line 3 $x_0=x_{s}$, and $x_{m+1}=x_{e}$, $x_1, x_2, \cdots , x_m$ are the midpoints of the cutlines $\epsilon_1,\epsilon_2,\cdots ,\epsilon_m$ respectively. Thereafter, Algorithm~\ref{alg:alg4} enters an iterative compression process for $f^*$ (lines 5-9). From line 9, it is evident that $f^*$ can be given by the sequence $\{x_0,x_1,\cdots ,x_m,x_{m+1}\}$. Thus lines 6-9 can be abstracted as map $\Psi : P(X_{free}) \to P(X_{free})$, and map $\Psi$ consists of $m$ submaps $\psi_1, \psi_2, \cdots,\psi_m$ (line 8), such that:
\begin{equation}
    \label{eq_IVD1}
    \Psi=\psi_m \circ \psi_{m-1} \circ \cdots \circ \psi_1.
\end{equation}
For any $\psi_k \in \{\psi_1, \psi_2, \cdots,\psi_m\}$ there is:
\begin{equation}
    \label{eq_IVD2}
    \begin{aligned}
         & l^{x_1}_{x_0}* \cdots *l^{x'_k}_{x_{k-1}}*l^{x_{k+1}}_{x'_k}* \cdots *l^{x_{n+1}}_{x_n}                   \\
         & = \psi_k \circ ( l^{x_1}_{x_0}* \cdots *l^{x_k}_{x_{k-1}}*l^{x_{k+1}}_{x_k}* \cdots *l^{x_{n+1}}_{x_n} ).
    \end{aligned}
\end{equation}
To simplify, let:
\begin{equation}
    \label{eq_IVD3}
    g = l^{x_1}_{x_0}* \cdots *l^{x_{k-1}}_{x_{k-2}}
    \quad \text{and} \quad
    h = l^{x_{k+2}}_{x_{k+1}}* \cdots *l^{x_{n+1}}_{x_n}.
\end{equation}
According to line 7 of the algorithm,
\begin{equation}
    \label{eq_IVD4}
    \begin{aligned}
        S(g*l^{x_k}_{x_{k-1}}*l^{x_{k+1}}_{x_k}*h) & = S(g) + S(l^{x_k}_{x_{k-1}}*l^{x_{k+1}}_{x_k}) + S(h)            \\
                                                   & \geqslant  S(g) + S(l^{x'_k}_{x_{k-1}}*l^{x_{k+1}}_{x'_k}) + S(h) \\
                                                   & = S(g*l^{x'_k}_{x_{k-1}}*l^{x_{k+1}}_{x'_k}*h)                    \\
                                                   & = S(\psi_k \circ (g*l^{x_k}_{x_{k-1}}*l^{x_{k+1}}_{x_k}*h)).
    \end{aligned}
\end{equation}
Hence, every map in $\{\psi_1, \psi_2, \cdots,\psi_m\}$ is compressed by $f^*$. According to (\ref{eq_IVD1}), the map $\Psi$ is also compressed for $f^*$.

Therefore, Algorithm~\ref{alg:alg4} can find the shortest path between input points $x_s$ and $x_e$ in the homotopy class $\mathrm{f}$. In addition, the solution to line 8 of the algorithm is simply given by the following: (i) If lines $l^{x_{k+1}}_{x_{k-1}}$ and $\epsilon_k$ intersect, $x_k$ is their intersection, and (ii) if lines $l^{x_{k+1}}_{x_{k-1}}$ and $\epsilon_k$ do not intersect, $x_k$ is the endpoint of $\epsilon_k$ that is closest to $l^{x_{k+1}}_{x_{k-1}}$.

According to Line 5 of Algorithm~\ref{alg:alg2}, we must execute Algorithm~\ref{alg:alg4} for each point $x_t$ on the segmentation line (line 10 of Algorithm~\ref{alg:alg2}), which is clearly time-consuming. However, for the neighboring points $x_{t'}$ of $x_t$, the locally optimal paths calculated using Algorithm~\ref{alg:alg4} for each CDT encoding $\mathrm{f}*\hat{\mathrm{f}}_{c_k}$ are significantly similar. Therefore, in practical programming, it is only necessary to run Algorithm~\ref{alg:alg4} once for $x_t$ at $t=0$. For the subsequent $x_t$, when running Algorithm~\ref{alg:alg4}, $l_{x_0}^{x_1} * l_{x_1}^{x_2} * \cdots * l_{x_n}^{x_{n+1}}$ in line 4 can be directly replaced with the locally optimal path of $x_{t-dt}$, which can significantly greatly reduce the number of iterations of the 5-10 lines. During testing, this optimization method can reduce the number of iterations to 2-4.

\begin{algorithm}[H]
    \caption{GetShortestPath.}\label{alg:alg4}
    \begin{algorithmic}[1]
        \STATE {\textbf{Input: }}$\mathrm{f},x_s,x_e$
        \STATE Use $\mathrm{f}$ to get the cutline sequence $\{\epsilon_1,\epsilon_2,\cdots ,\epsilon_m\}$
        \STATE $\{x_0,x_1,\cdots ,x_m,x_{m+1}\} \gets \{x_s,\epsilon_1(\frac{1}{2}),\cdots ,\epsilon_m(\frac{1}{2}),x_e\}$
        \STATE $f^* \gets l_{x_0}^{x_1} * l_{x_1}^{x_2} * \cdots * l_{x_n}^{x_{n+1}}$
        \STATE \textbf{repeat}
        \STATE \hspace{0.5cm} ${cost} \gets S(f^*) $
        \STATE \hspace{0.5cm} \textbf{for $k \in \mathbb{N}^m_1$ do}
        \STATE \hspace{1.0cm} $x_k = \mathop {\arg \min }\limits_{x \in \epsilon_k}S(l_{x_{k-1}}^x * l_x^{x_{k+1}})$
        \STATE \hspace{0.5cm} $f^* \gets l_{x_0}^{x_1} * l_{x_1}^{x_2} * \cdots * l_{x_n}^{x_{n+1}}$
        \STATE \textbf{until} $S(f^*) - cost < \varepsilon $
        \STATE \textbf{return} $f^*$
    \end{algorithmic}
    \label{alg4}
\end{algorithm}

\section{RESULTS}
In this section, we present the results of the several simulation experiments that were conducted to demonstrate the effectiveness and efficiency of CDT-Dijkstra. An Intel NUC (next unit of computing) computer was used, which was an Intel i7-1165G7 (4.7GHz) with 32 GB of RAM. For the simulation, we used C++ on Ubuntu 20.04. As shown in Fig.~\ref{fig_map}, we chose three regular environments (cluttered, trap, and maze) and a challenging mega map (2D map of the German Museum) to test our algorithm in this experiment. Fig.~\ref{fig_map}(e)-(h) demonstrate the results of the convex dissection of the four test environments using the method in \cite{arxivLiu2023homotopy}; details of the initialized construction of the topological space $X_{tp}$ are listed in \textbf{Table~\ref{tab:table1}}.

\begin{table}
    \begin{center}
        \caption{Initialization of Four Maps\label{tab:table1}}
        \centering
        \begin{tabular}{|c|c|c|c|c|}
            \hline
            \makecell[c]{Map                                        \\Name} & \makecell[c]{Map Image\\Resolution} & \makecell*[c]{Cutlines} & \makecell*[c]{Convex\\Polygons} & \makecell*[c]{Time (ms)}\\
            \hline
            \makecell[c]{Cluttered} & $1000*1000$ & $135$ & $122$ & $19.7\pm2.5$   \\
            \hline
            \makecell[c]{Trap}      & $1000*1000$ & $38$  & $37$  & $9.2\pm3.3$    \\
            \hline
            \makecell[c]{Maze}      & $1000*1000$ & $95$  & $95$  & $15.4\pm3.6$   \\
            \hline
            \makecell[c]{Museum}    & $9706*9706$ & $646$ & $608$ & $397.2\pm18.9$ \\
            \hline
        \end{tabular}
    \end{center}
\end{table}

CDT-Dijkstra is a unique path-planning algorithm that consists of two parts: SetInit and GetGoal. To demonstrate that this algorithm achieves the effect of needing one run to obtain the optimal path from all points to the initial point, we set up four planning tasks in each environment, where the optimal path was planned from the initial point S to goal points 1-4, respectively. In the experiment, CDT-Dijkstra executed SetInit\footnote{In the experiment, we discretized the cutlines with an interval of two pixels in length.} once for the initial point S in each environment, and subsequently used GetGoal to obtain the optimal path for each goal point separately; the time spent on SetInit and GetGoal was recorded. This operation was performed 100 times for each environment. The average values of SetInit and GetGoal times for each task are shown in Fig.~\ref{fig_SetGet}.

\begin{figure}[!t]
    \centering
    \includegraphics[width=3.4in]{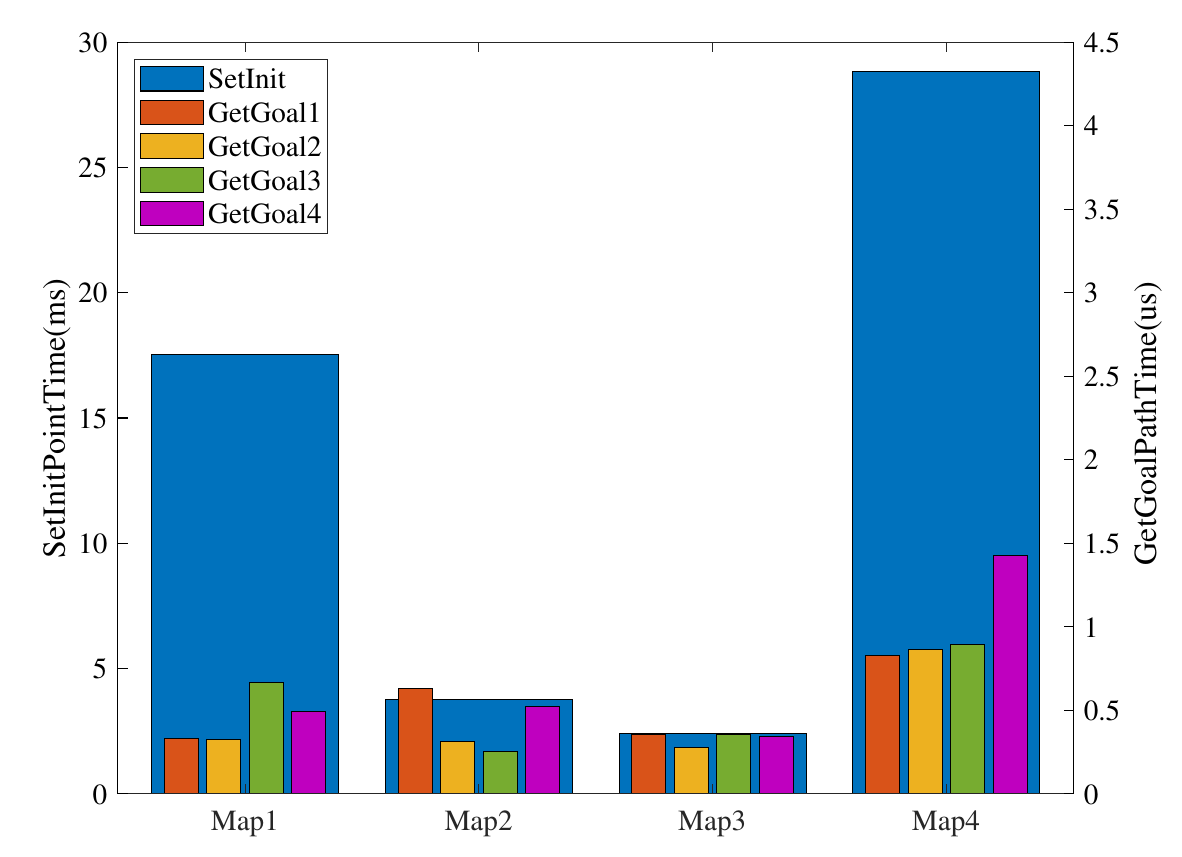}
    \caption{CDT-Dijkstra test results. The timescales of the SetInit stage and GetGoal stages are on the left and right, respectively.}
    \label{fig_SetGet}
\end{figure}

Based on the experiments, the SetInit stage of CDT-Dijkstra consumed the majority of the algorithm time; however, it maintains a high level of real-time performance (the average usage was 17.5 ms, 3.7 ms, 2.4 ms, and 28.8 ms). Moreover, the time used by the GetGoal stage remains in the sub-microsecond range, even for complex tasks such as Task 4 in Fig.~\ref{fig_map}(d), where the time used by GetGoal does not exceed 1.5 us; these times are almost negligible. Therefore, we consider that CDT-Dijkstra achieves a fast search for the optimal path from one point to all the other points in a continuous space.

To determine the advantages and disadvantages of CDT-Dijkstra compared to general optimal path planning algorithms, we present a comparison of the results between the proposed CDT-Dijkstra algorithm and the following four advanced algorithms in the OMPL open source framework: RRT* algorithm, PRM* algorithm \cite{karaman2011sampling}, Informed-RRT* \cite{gammell2018informed}, and advanced batch informed trees algorithm (ABIT*) \cite{strub2020advanced}. Two parameters were used to compare the performance of the algorithms $t_{init}$, including the time of the initial solution and $t_{1\%}$, which is the time required to determine a solution of cost $1.01 \times Cost_{opt}$, where $Cost_{opt}$ is the optimal cost. These four algorithms were executed 100 times for each task, as shown in Fig.~\ref{fig_map}(a)-(h), and the mean values of $t_{init}$ and $t_{1\%}$ for these algorithms for each task are shown in Fig.~\ref{fig_map}(i)-(l), respectively. For a better comparison with other algorithms, the completion times of the CDT-Dijkstra algorithms of SetInit and AllInit are indicated by the blue and red lines, respectively, where AllInit is equal to the SetInit time plus the corresponding map initialization\footnote{In practice, a map only needs to be initialized once. The initialisation result can be saved as a file that can be directly accessed for future use.} time show in \textbf{Table~\ref{tab:table1}}.

\begin{figure*}[!t]
    \centering
    \subfloat[]{\includegraphics[width=1.0in]{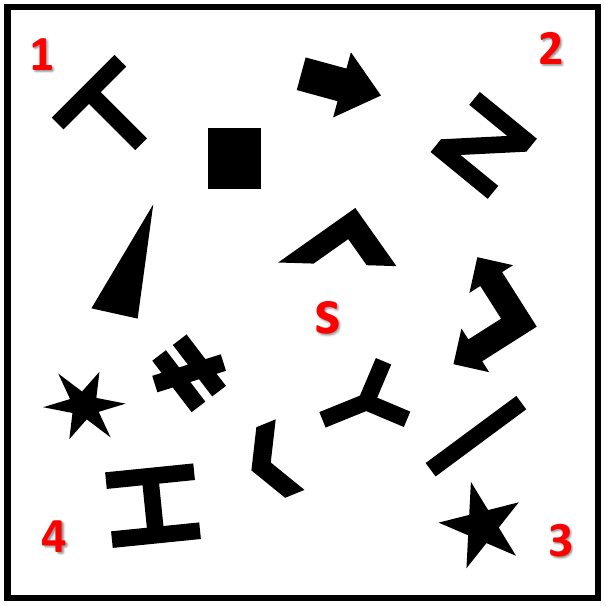}
        \label{fig_map:map1}}
    \hfil
    \subfloat[]{\includegraphics[width=1.0in]{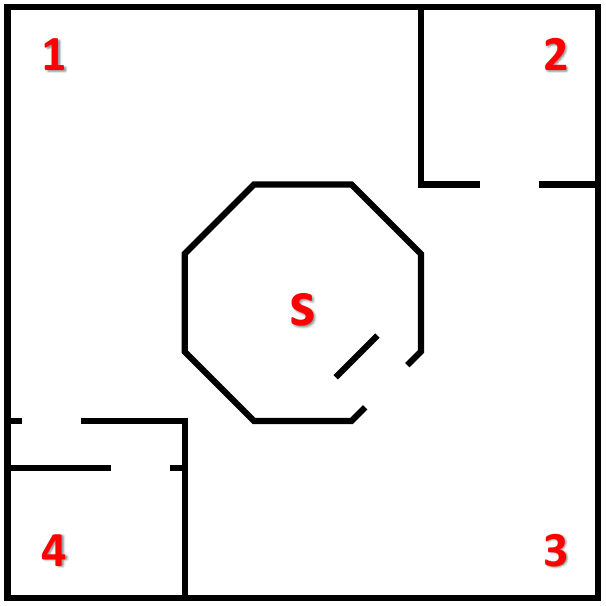}
        \label{fig_map:map2}}
    \hfil
    \subfloat[]{\includegraphics[width=1.0in]{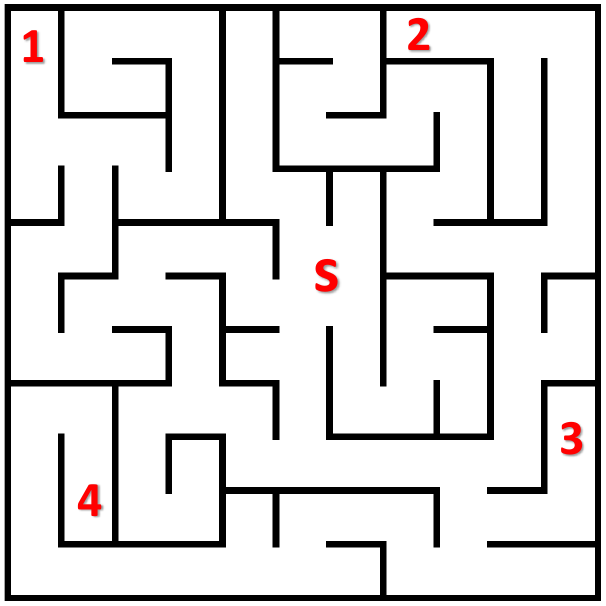}
        \label{fig_map:map3}}
    \hfil
    \subfloat[]{\includegraphics[width=1.0in]{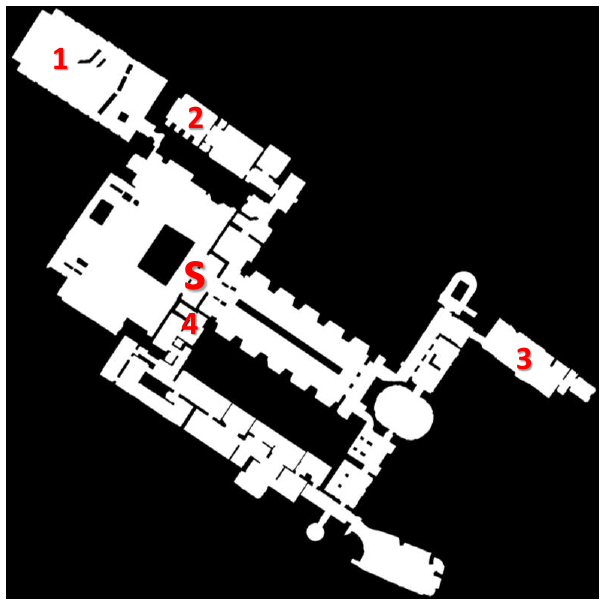}
        \label{fig_map:map4}}
    \\
    \subfloat[]{\includegraphics[width=1.0in]{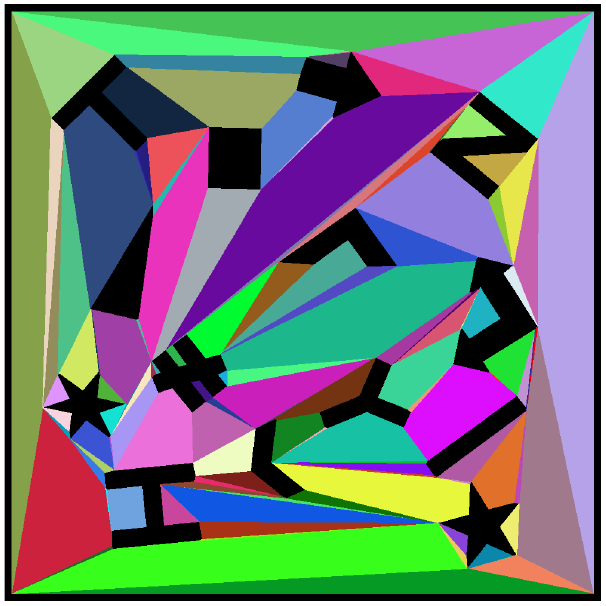}
        \label{fig_map:init1}}
    \hfil
    \subfloat[]{\includegraphics[width=1.0in]{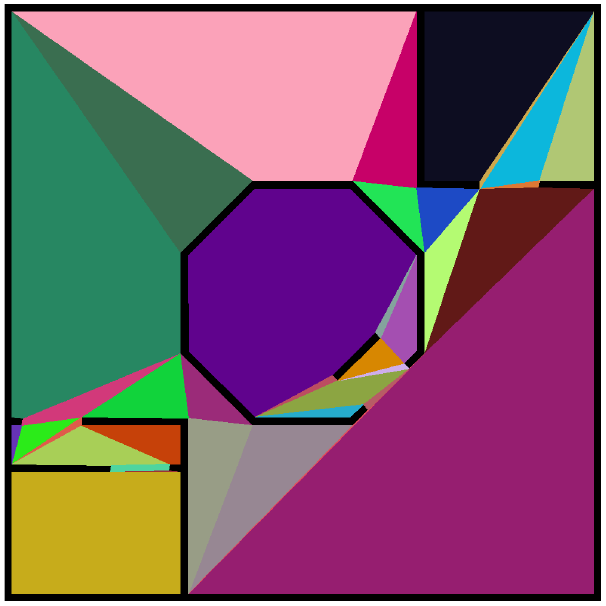}
        \label{fig_map:init2}}
    \hfil
    \subfloat[]{\includegraphics[width=1.0in]{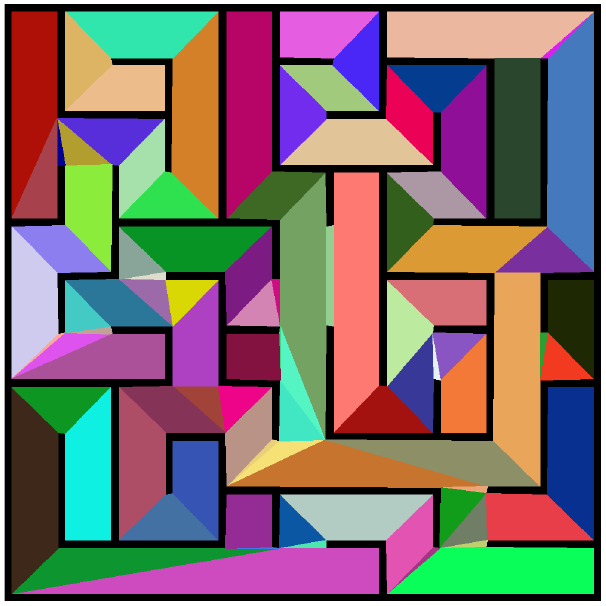}
        \label{fig_map:init3}}
    \hfil
    \subfloat[]{\includegraphics[width=1.0in]{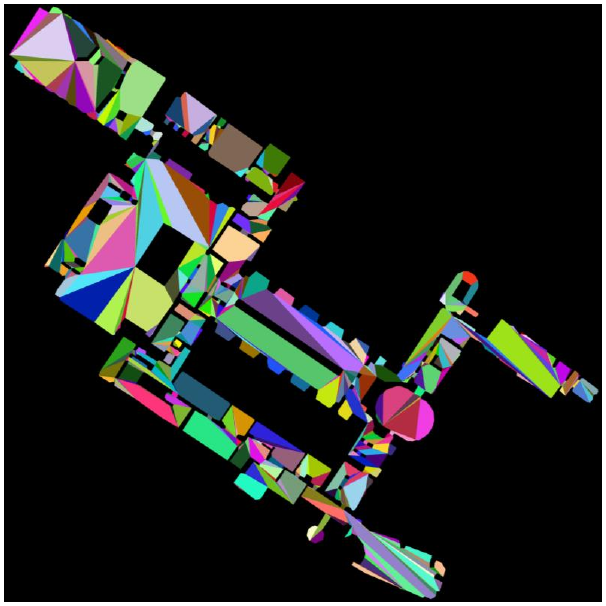}
        \label{fig_map:init4}}
    \\
    \subfloat[]{\includegraphics[width=1.7in]{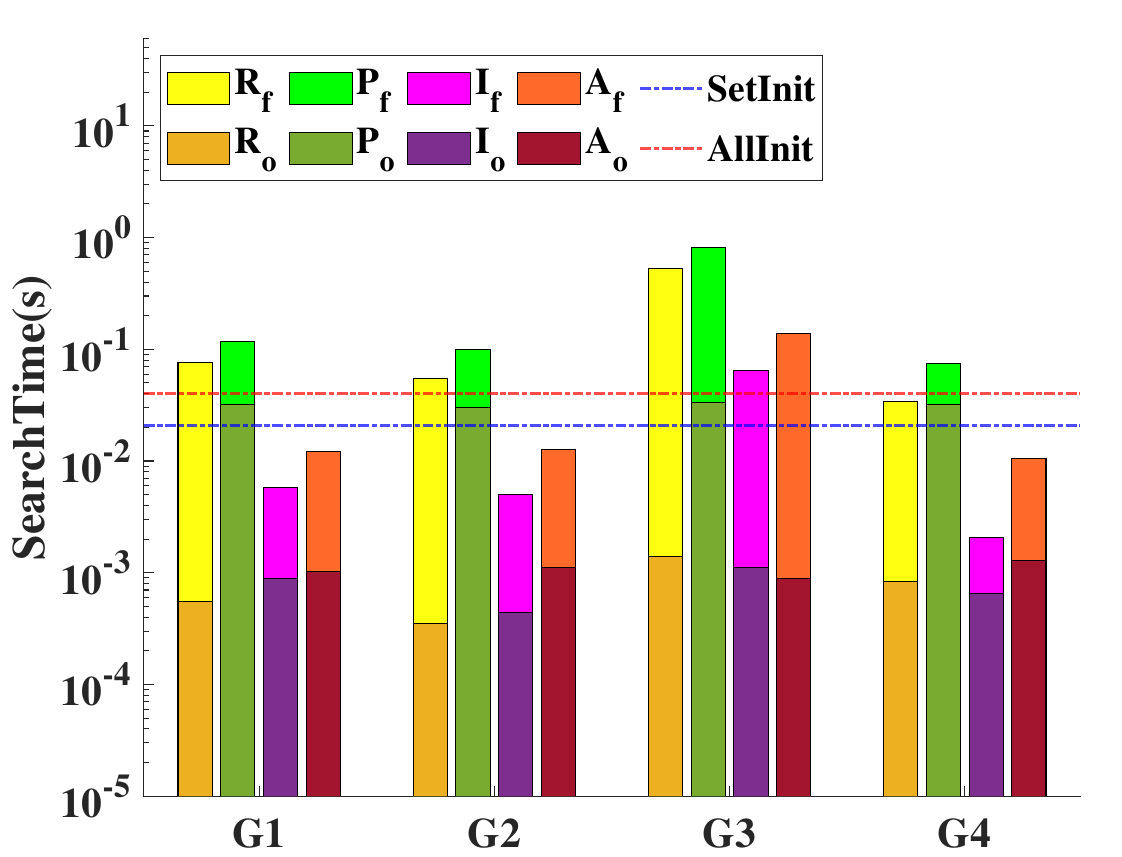}
        \label{fig_map:test1}}
    \hfil
    \subfloat[]{\includegraphics[width=1.7in]{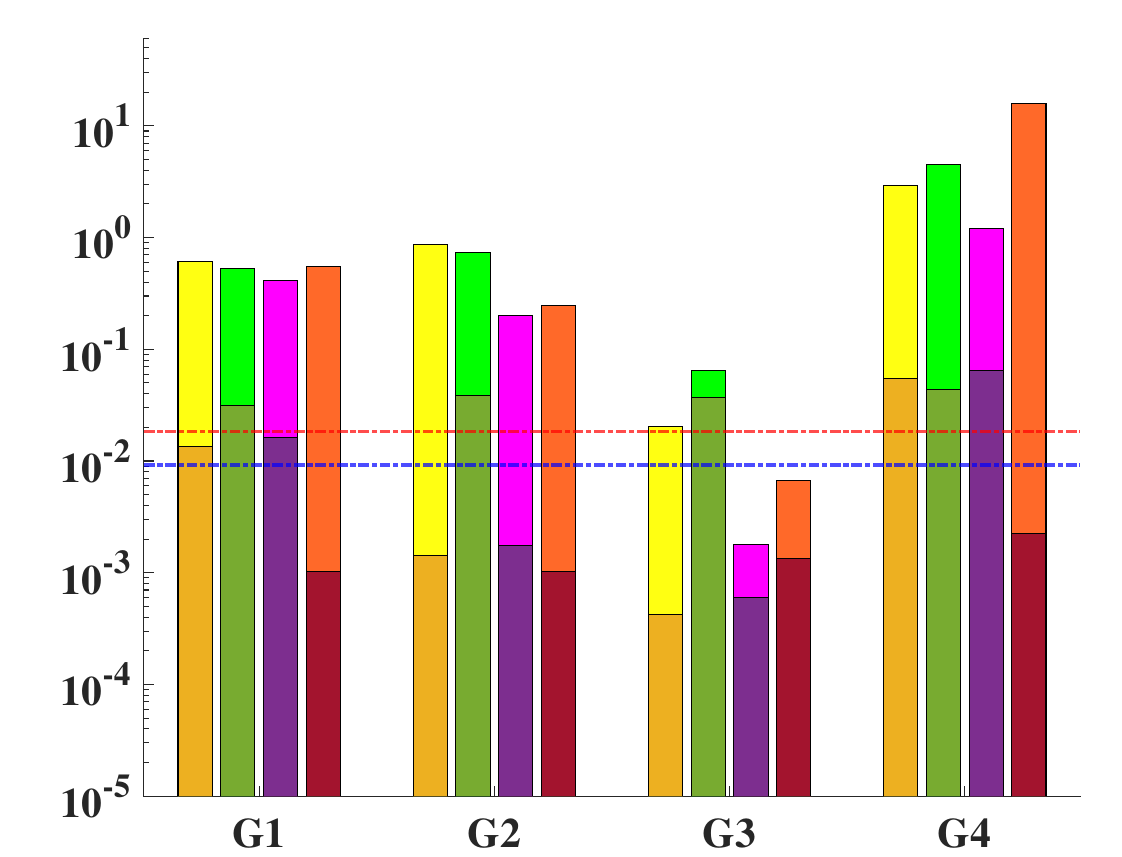}
        \label{fig_map:test2}}
    \hfil
    \subfloat[]{\includegraphics[width=1.7in]{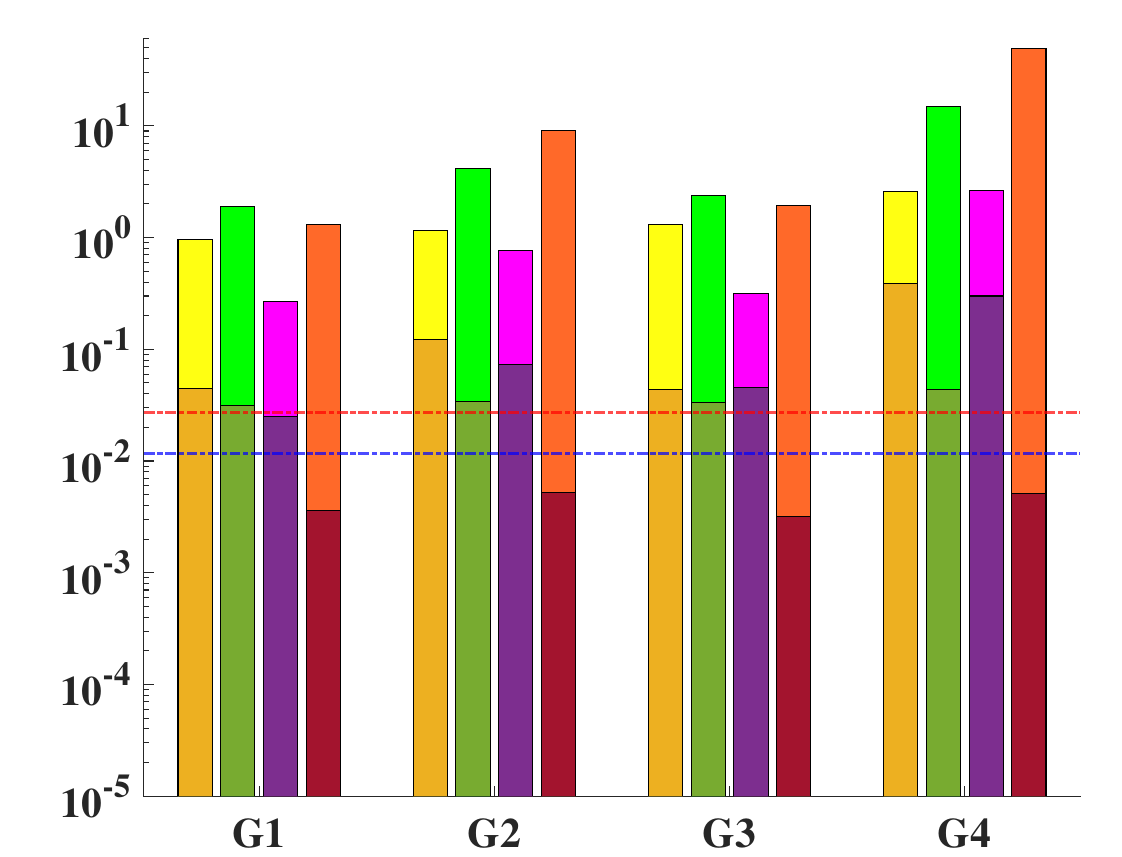}
        \label{fig_map:test3}}
    \hfil
    \subfloat[]{\includegraphics[width=1.7in]{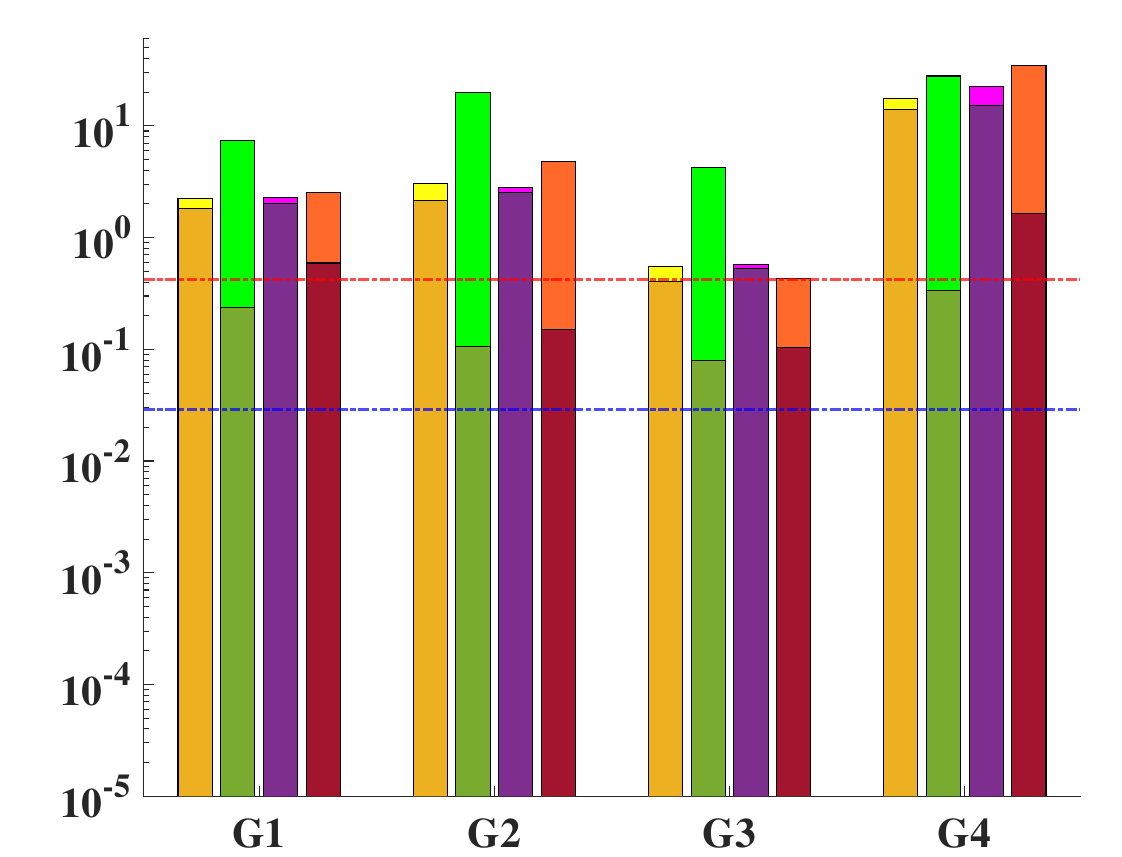}
        \label{fig_map:test4}}
    \caption{Illustration of the experiment map and planning tasks. (a) Cluttered environment, (b) trap, (c) maze and (d) 2D map of the German Museum. Each environment contains four planning tasks starting from S and ending G. (e)-(h): convex division results of each environment. (i)-(l): performance statistics of each planner in the comparative experiment. The completion times of the CDT-Dijkstra algorithms of SetInit and AllInit are indicated by the blue and red lines, respectively, where AllInit is equal to the SetInit time plus the corresponding map initializatio.}
    \label{fig_map}
\end{figure*}
\begin{table}
    \caption{The Abbreviation for Each Criteria Type for Each Algorithm\label{tab:table2}}
    \centering
    \begin{tabular}{|c|c|c|c|c|}
        \hline
        Criteria   & RRT*           & PRM*           & Informed-RRT*  & ABIT*          \\
        \hline
        $t_{init}$ & $R_\mathrm{f}$ & $P_\mathrm{f}$ & $I_\mathrm{f}$ & $A_\mathrm{f}$ \\
        \hline
        $t_{1\%}$  & $R_\mathrm{o}$ & $P_\mathrm{o}$ & $I_\mathrm{o}$ & $A_\mathrm{o}$ \\
        \hline
    \end{tabular}
\end{table}

The experimental results demonstrate that CDT-Dijkstra not only plans the optimal path from all the points to the starting point by running SetInit once, but also has a significant advantage over other algorithms when managing complex tasks. The planning speed of CDT-Dijkstra is one to three orders of magnitude faster than the other four algorithms in the maze and museum environments, especially in the museum environment, where the other algorithms require 10-50 times more time than CDT-Dijkstra to find the initial path; CDT-Dijkstra also performs better in the trap environment. However, based on the experimental results, we also found that CDT-Dijkstra performed relatively poorly in certain simple tasks, such as cluttered environments and Task 3 in trap environments. These phenomena are related to the working principle of the CDT-Dijkstra. When calculating the CDT encoding set of cutlines, CDT-Dijkstra is guided by the topological graph $X_{tp}$, and updates in an orderly manner from the cutline near the initial point to the outside. Therefore, CDT-Dijkstra is not affected by the twists and turns of free space and the narrow gaps of the traps during operation. On the other hand, sampling-based optimal path planning methods are more easily trapped. In the trap environment shown in Fig.~\ref{fig_map}, the sampling-based methods encounter significantly fewer challenges in Task 3 compared to several other tasks, thus they perform better in Task 3 and plan faster than CDT-Dijkstra.

The issue reflected by the experimental data show in Fig.~\ref{fig_map}(i) must be analyzed in conjunction with the number and distribution of independent obstacles in the experimental environment. Here, an independent obstacle refers to the subset of obstacles in the obstacle space $X_{obs}$ that are surrounded by the free space $X_{free}$. The number of independent obstacles directly determines the number of homotopy classes of paths between any two connected points in free space; this relationship is exponential with base $2$. CDT-Dijkstra can remove the most invalid homotopy classes by orderly updating outward from the cutline near the initial point when calculating the CDT encoding set of the cutline. However, when several independent obstacles are uniformly distributed near the initial point, CDT-Dijkstra will waste significant time in Algorithm~\ref{alg:alg3} because, when the computation of the CDT encoding set of a new cutline is completed, the CDT encoding set of the existing adjacent cutlines is easily changed; that is, the condition in line 6 of Algorithm~\ref{alg:alg3} is more likely to be triggered. The data presented in \textbf{Table~\ref{tab:table3}} also support the aforementioned, with the trap and maze environments having a relatively small number of independent obstacles and $\xi_{c_r} \neq \xi_{temp}$ being triggered significantly less often than the other two environments. Comparing the cluttered and museum environments, we find that although the museum has a higher number of independent obstacles, the independent obstacles in cluttered environments are more evenly distributed; therefore, $\xi_{c_r} \neq \xi_{temp}$ is triggered significantly more often in a cluttered environment than in a museum. Therefore, in cluttered environments, CDT-Dijkstra has a slightly weaker performance than other sampling-based methods when performing optimal path planning tasks between two points.

\begin{table}
    \caption{Effect of the Number of Independent Obstacles in the Map\label{tab:table3}}
    \centering
    \begin{tabular}{|c|c|c|c|c|}
        \hline
        \makecell*[c]{Maps} & Cluttered & Trap & Maze & Museum \\
        \hline
        \makecell[c]{Number of                                \\ Independent Obstacles}  & $14$     & $2$  & $1$  & $38$        \\
        \hline
        \makecell[c]{Number of                                \\ Homotopy Classes}  & $16384$   & $4$  & $2$  & $2^{38}$ \\
        \hline
        \makecell[c]{Number of triggers                       \\ for $\xi_{c_r} \neq \xi_{temp}$} & $5437$    & $44$ & $32$ & $962$       \\
        \hline
    \end{tabular}
\end{table}

\section{CONCLUSIONS}
In this study, based on the theoretical framework of the CDT encoder, we propose a planning principle that only considers the points on the cutline, thereby reducing the state space in the path-planning task from 2D to 1D. Based on this principle, for a given $x_{init}$, only the optimal CDT encoder set of all the cutlines can be obtained, and the global optimal path from any $x_{goal}$ to $x_{init}$ can be easily obtained. Based on the aforementioned work, we propose a unique optimal path-planning algorithm called CDT-Dijkstra. Guided by the topological graph $X_{tp}$, CDT-Dijkstra completes the calculation of the CDT encoding set of each cutline in an orderly manner from the cutline near the initial point outward, subsequently returning at an extremely fast global optimal path to any goal point. The CDT-Dijkstra algorithm implements a fast computation of the global optimal path from one point to all the other points in a 2D continuous space. In addition, we also proposed a method for rapidly finding the optimal path in a homotopy class (Algorithm~\ref{alg:alg4}), proved the correctness of this method through theories, and provided an optimization scheme for this method in practical applications. Finally, by testing in a series of environments, the experimental results demonstrate that CDT-Dijkstra not only plans the optimal path from all points at once, but also has a significant advantage over other algorithms when encountering certain complex tasks, that is, CDT-Dijkstra is not affected by the twists and turns of free space and the narrow gaps of traps during operation.

In practical robotics applications, CDT-Dijkstra can be run in a dual-threaded manner, that is, one thread runs SetInit according to the target position of the robot, and the other thread uses GetGoal to update the optimal path of the current position of the robot in real-time during the robot movement. CDT-Dijkstra is ideal for application to tasks in which a robot is required to move from a fixed starting point to a location in space and return; for example, inspection and food delivery robots.

\bibliographystyle{IEEEtran}
\bibliography{IEEEabrv,uref}

\begin{thebibliography}{10}
\providecommand{\url}[1]{#1}
\csname url@samestyle\endcsname
\providecommand{\newblock}{\relax}
\providecommand{\bibinfo}[2]{#2}
\providecommand{\BIBentrySTDinterwordspacing}{\spaceskip=0pt\relax}
\providecommand{\BIBentryALTinterwordstretchfactor}{4}
\providecommand{\BIBentryALTinterwordspacing}{\spaceskip=\fontdimen2\font plus
\BIBentryALTinterwordstretchfactor\fontdimen3\font minus
  \fontdimen4\font\relax}
\providecommand{\BIBforeignlanguage}[2]{{%
\expandafter\ifx\csname l@#1\endcsname\relax
\typeout{** WARNING: IEEEtran.bst: No hyphenation pattern has been}%
\typeout{** loaded for the language `#1'. Using the pattern for}%
\typeout{** the default language instead.}%
\else
\language=\csname l@#1\endcsname
\fi
#2}}
\providecommand{\BIBdecl}{\relax}
\BIBdecl

\bibitem{wang2020overview}
L.~Wang, S.~Liu, H.~Liu, and X.~V. Wang, ``Overview of human-robot
  collaboration in manufacturing,'' in \emph{Proceedings of 5th International
  Conference on the Industry 4.0 Model for Advanced Manufacturing: AMP
  2020}.\hskip 1em plus 0.5em minus 0.4em\relax Springer, 2020, pp. 15--58.

\bibitem{zhang2018multilevel}
X.~Zhang, J.~Wang, Y.~Fang, and J.~Yuan, ``Multilevel humanlike motion planning
  for mobile robots in complex indoor environments,'' \emph{IEEE Transactions
  on Automation Science and Engineering}, vol.~16, no.~3, pp. 1244--1258, 2018.

\bibitem{bopardikar2015multiobjective}
S.~D. Bopardikar, B.~Englot, and A.~Speranzon, ``Multiobjective path planning:
  Localization constraints and collision probability,'' \emph{IEEE Transactions
  on Robotics}, vol.~31, no.~3, pp. 562--577, 2015.

\bibitem{duhautbout2022efficient}
T.~Duhautbout, R.~Talj, V.~Cherfaoui, F.~Aioun, and F.~Guillemard, ``Efficient
  speed planning in the path-time space for urban autonomous driving,'' in
  \emph{2022 IEEE 25th International Conference on Intelligent Transportation
  Systems (ITSC)}.\hskip 1em plus 0.5em minus 0.4em\relax IEEE, 2022, pp.
  1268--1274.

\bibitem{tzafestas2018mobile}
S.~G. Tzafestas, ``Mobile robot control and navigation: A global overview,''
  \emph{Journal of Intelligent \& Robotic Systems}, vol.~91, no.~1, pp. 35--58,
  2018.

\bibitem{naderi2015rt}
K.~Naderi, J.~Rajam{\"a}ki, and P.~H{\"a}m{\"a}l{\"a}inen, ``Rt-rrt* a
  real-time path planning algorithm based on rrt,'' in \emph{Proceedings of the
  8th ACM SIGGRAPH Conference on Motion in Games}, 2015, pp. 113--118.

\bibitem{do2022heat}
H.~Do, A.~V. Le, L.~Yi, J.~C.~C. Hoong, M.~Tran, P.~Van~Duc, M.~B. Vu,
  O.~Weeger, and R.~E. Mohan, ``Heat conduction combined grid-based
  optimization method for reconfigurable pavement sweeping robot path
  planning,'' \emph{Robotics and Autonomous Systems}, vol. 152, p. 104063,
  2022.

\bibitem{zhao2022surgical}
Y.~Zhao, Y.~Wang, J.~Zhang, X.~Liu, Y.~Li, S.~Guo, X.~Yang, and S.~Hong,
  ``Surgical gan: Towards real-time path planning for passive flexible tools in
  endovascular surgeries,'' \emph{Neurocomputing}, 2022.

\bibitem{dijkstra1959note}
E.~W. Dijkstra, ``A note on two problems in connexion with graphs:(numerische
  mathematik, 1 (1959), p 269-271),'' 1959.

\bibitem{hart1968formal}
P.~E. Hart, N.~J. Nilsson, and B.~Raphael, ``A formal basis for the heuristic
  determination of minimum cost paths,'' \emph{IEEE transactions on Systems
  Science and Cybernetics}, vol.~4, no.~2, pp. 100--107, 1968.

\bibitem{lozano1979algorithm}
T.~Lozano-P{\'e}rez and M.~A. Wesley, ``An algorithm for planning
  collision-free paths among polyhedral obstacles,'' \emph{Communications of
  the ACM}, vol.~22, no.~10, pp. 560--570, 1979.

\bibitem{lavalle2001randomized}
S.~M. LaValle and J.~J. Kuffner~Jr, ``Randomized kinodynamic planning,''
  \emph{The international journal of robotics research}, vol.~20, no.~5, pp.
  378--400, 2001.

\bibitem{karaman2011sampling}
S.~Karaman and E.~Frazzoli, ``Sampling-based algorithms for optimal motion
  planning,'' \emph{The international journal of robotics research}, vol.~30,
  no.~7, pp. 846--894, 2011.

\bibitem{gammell2018informed}
J.~D. Gammell, T.~D. Barfoot, and S.~S. Srinivasa, ``Informed sampling for
  asymptotically optimal path planning,'' \emph{IEEE Transactions on Robotics},
  vol.~34, no.~4, pp. 966--984, 2018.

\bibitem{gammell2020batch}
------, ``Batch informed trees (bit*): Informed asymptotically optimal anytime
  search,'' \emph{The International Journal of Robotics Research}, vol.~39,
  no.~5, pp. 543--567, 2020.

\bibitem{strub2022adaptively}
M.~P. Strub and J.~D. Gammell, ``Adaptively informed trees (ait*) and effort
  informed trees (eit*): Asymmetric bidirectional sampling-based path
  planning,'' \emph{The International Journal of Robotics Research}, vol.~41,
  no.~4, pp. 390--417, 2022.

\bibitem{kloetzer2015optimizing}
M.~Kloetzer, C.~Mahulea, and R.~Gonzalez, ``Optimizing cell decomposition path
  planning for mobile robots using different metrics,'' in \emph{2015 19th
  international conference on system theory, control and computing
  (ICSTCC)}.\hskip 1em plus 0.5em minus 0.4em\relax IEEE, 2015, pp. 565--570.

\bibitem{li2020new}
Z.~Li, Z.~Zhang, H.~Liu, and L.~Yang, ``A new path planning method based on
  concave polygon convex decomposition and artificial bee colony algorithm,''
  \emph{International Journal of Advanced Robotic Systems}, vol.~17, no.~1, p.
  1729881419894787, 2020.

\bibitem{gonzalez2017comparative}
R.~Gonzalez, M.~Kloetzer, and C.~Mahulea, ``Comparative study of trajectories
  resulted from cell decomposition path planning approaches,'' in \emph{2017
  21st International Conference on System Theory, Control and Computing
  (ICSTCC)}.\hskip 1em plus 0.5em minus 0.4em\relax IEEE, 2017, pp. 49--54.

\bibitem{chi2021generalized}
W.~Chi, Z.~Ding, J.~Wang, G.~Chen, and L.~Sun, ``A generalized voronoi
  diagram-based efficient heuristic path planning method for rrts in mobile
  robots,'' \emph{IEEE Transactions on Industrial Electronics}, vol.~69, no.~5,
  pp. 4926--4937, 2021.

\bibitem{huang2021path}
S.-K. Huang, W.-J. Wang, and C.-H. Sun, ``A path planning strategy for
  multi-robot moving with path-priority order based on a generalized voronoi
  diagram,'' \emph{Applied Sciences}, vol.~11, no.~20, p. 9650, 2021.

\bibitem{munkres2018elements}
J.~R. Munkres, \emph{Elements of algebraic topology}.\hskip 1em plus 0.5em
  minus 0.4em\relax CRC press, 2018.

\bibitem{bhattacharya2012search}
S.~Bhattacharya, M.~Likhachev, and V.~Kumar, ``Search-based path planning with
  homotopy class constraints in 3d,'' in \emph{Proceedings of the AAAI
  Conference on Artificial Intelligence}, vol.~26, no.~1, 2012, pp. 2097--2099.

\bibitem{mccammon2021topological}
S.~McCammon and G.~A. Hollinger, ``Topological path planning for autonomous
  information gathering,'' \emph{Autonomous Robots}, vol.~45, pp. 821--842,
  2021.

\bibitem{wakulicz2023topological}
J.~Wakulicz, K.~M.~B. Lee, T.~Vidal-Calleja, and R.~Fitch, ``Topological
  trajectory prediction with homotopy classes,'' \emph{arXiv preprint
  arXiv:2301.09821}, 2023.

\bibitem{arxivLiu2023homotopy}
\BIBentryALTinterwordspacing
J.~Liu, M.~Fu, A.~Liu, W.~Zhang, B.~Chen, R.~Prakapovich, and U.~Sychou,
  ``Homotopy path class encoder based on convex dissection topology,'' 2023.
  [Online]. Available: \url{https://arxiv.org/abs/2302.13026}
\BIBentrySTDinterwordspacing

\bibitem{strub2020advanced}
M.~P. Strub and J.~D. Gammell, ``Advanced bit*(abit*): Sampling-based planning
  with advanced graph-search techniques,'' in \emph{2020 IEEE International
  Conference on Robotics and Automation (ICRA)}.\hskip 1em plus 0.5em minus
  0.4em\relax IEEE, 2020, pp. 130--136.

\end{thebibliography}

\end{document}